\documentclass{article}

\usepackage{microtype}
\usepackage{graphicx}
\usepackage{booktabs} 
\usepackage{color}
\usepackage{subcaption}
\usepackage{xcolor}
\usepackage{colortbl}
\usepackage{hyperref}

\usepackage[accepted]{icml2024-diffxyz}

\usepackage{amsmath}
\usepackage{amssymb}
\usepackage{mathtools}
\usepackage{amsthm}

\usepackage[capitalize,noabbrev]{cleveref}

\theoremstyle{plain}
\newtheorem{theorem}{Theorem}[section]

\theoremstyle{definition}
\newtheorem{definition}[theorem]{Definition}

\newtheorem{example}[theorem]{Example}
\theoremstyle{remark}
\newtheorem{remark}[theorem]{Remark}

\definecolor{lightgray}{gray}{0.95}

\renewcommand{\epsilon}{\varepsilon}
\renewcommand{\phi}{\varphi}
\newcommand*{\N}{\mathbb{N}}

\newcommand*{\R}{\mathbb{R}}

\newcommand{\norm}[1]{\left \| #1 \right \|}
\newcommand{\abs}[1]{\left|#1 \right|}
\newcommand{\skal}[1]{\left \langle #1 \right \rangle}

\DeclareMathOperator{\argmin}{argmin}

\DeclareMathOperator{\dom}{dom}

\theoremstyle{definition}
\newtheorem{defi}{Definition}[section]

\theoremstyle{plain}
\newtheorem{thm}[defi]{Theorem}

\usepackage{mathtools}

\numberwithin{equation}{section}

\usepackage[textsize=tiny]{todonotes}

\usepackage[
backend=biber,
style=ieee,
citestyle=ieee,
]{biblatex}
\AtBeginBibliography{\small}

\icmltitlerunning{(Almost) Smooth Sailing}

\begin{document}

\twocolumn[
\icmltitle{(Almost) Smooth Sailing: Towards Numerical Stability of Neural Networks Through Differentiable Regularization of the Condition Number}



\icmlsetsymbol{equal}{*}
\begin{icmlauthorlist}
\icmlauthor{Rossen Nenov}{equal,yyy,comp}
\icmlauthor{Daniel Haider}{equal,yyy}
\icmlauthor{Peter Balazs}{yyy}
\end{icmlauthorlist}

\begin{icmlauthorlist}

\end{icmlauthorlist}

\icmlaffiliation{yyy}{Acoustics Research Institute, Austrian Academy of Sciences, Vienna, Austria}
\icmlaffiliation{comp}{Department of Mathematics, University of Vienna, Vienna, Austria}

\icmlcorrespondingauthor{Rossen Nenov}{rossen.nenov@oeaw.ac.at}
\icmlcorrespondingauthor{Daniel Haider}{daniel.haider@oeaw.ac.at}



\icmlkeywords{Numerical Stability, Regularization, Condition Number, Differentiability, Denoising}

\vskip 0.3in]



\printAffiliationsAndNotice{\icmlEqualContribution} 

\begin{abstract}
Maintaining numerical stability in machine learning models is crucial for their reliability and performance. One approach to maintain stability of a network layer is to integrate the condition number of the weight matrix as a regularizing term into the optimization algorithm. However, due to its discontinuous nature and lack of differentiability the condition number is not suitable for a gradient descent approach. This paper introduces a novel regularizer that is provably differentiable almost everywhere and promotes matrices with low condition numbers.
In particular, we derive a formula for the gradient of this regularizer which can be easily implemented and integrated into existing optimization algorithms. We show the advantages of this approach for noisy classification and denoising of MNIST images.

\end{abstract}

\section{Introduction}








Numerical stability in neural networks refers to the sensitivity of model predictions and training dynamics to small perturbations in the input data, model parameters, or other computational operations. Good stability properties offer significant benefits, such as consistency enhancement and robustness in predictions, thereby improving generalization and interpretability \cite{Stability}.
Instabilities can arise due to several factors, such as the choice of activation functions, initialization, optimization hyperparameters, and quantization effects during training and inference. Several regularization methods have been developed to intercept this, e.g. dropout, lasso, randomness, etc. \cite{goodfellow2016deeplearning,zou1005regularization,hagemann2020StabilizingIN}.
Differentiability is essential here as it enables the incorporation of such methods into gradient-based optimization algorithms, hence allowing to gradually regularize the numerical stability during training.

\begin{figure}[t]
    \centering
    \includegraphics[width=0.9\columnwidth]{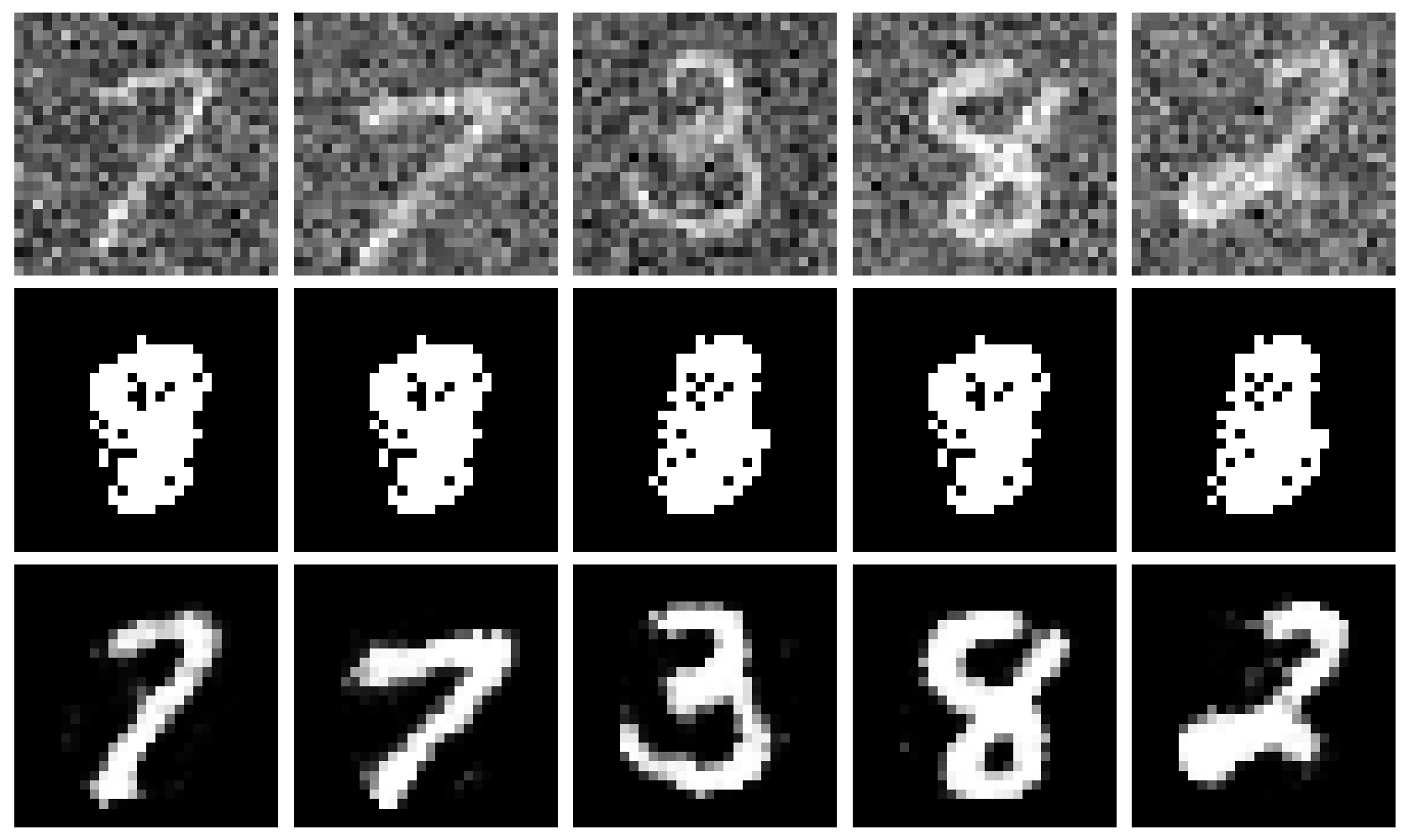}
    \caption{Results of MNIST denoising with autoencoders. Top: MNIST images with added Gaussian noise. Mid: No regularization. Bottom: Proposed regularization. While the vanilla autoencoder struggles significantly, the regularized one performs well.}
    \label{fig:denoising}
\end{figure}

In this paper, we focus on the numerical stability of a neural network by means of their weight matrices, and how to maintain it through regularization. In numerical linear algebra, a way to measure the numerical stability of a matrix $S$ is via its condition number. For invertible matrices $S$ it is defined as $\norm{S}_2\norm{S^{-1}}_2$, where $\norm{\cdot}_2$ is the matrix operator norm induced by the Euclidean norm \cite{Golub}. If $S$ is non-invertible or non-square, alternative definitions of a condition number have been proposed. In the context of least squares minimization a particularly common one is 
\begin{align} \label{kappa}
    \kappa (S):={\norm{S}}_2{\norm{S^\dagger}}_2,
\end{align}
where $S^\dagger $ is the Moore-Penrose inverse, or pseudoinverse \cite{Cond2}.
In this sense, a matrix is considered optimally conditioned if $\kappa(S) = 1$. Existing approaches to promote this kind of stability are based on adaptive optimization techniques \cite{hasannasab2020parseval}, iterative projections \cite{cisse2017parseval}, or adding $\kappa$ directly to the learning objective as a regularizing term \cite{balazs2024stableencoders}. In the latter approach, minimizing $\kappa$ via gradient-based optimization has been shown to work well. However, it should be noted that the function $\kappa:\R^{n\times m}\rightarrow \R$ is discontinuous and, therefore not differentiable. 

To address this issue, this paper introduces an alternative quantity as regularizer that achieves an optimal condition number, guarantees full rank, and is differentiable almost everywhere.

This manuscript is structured as follows. After this introduction we introduce the proposed regularizer and prove that its minimization is equivalent to minimizing the condition number $\kappa$. Subsequently, we prove that the regularizer is differentiable almost everywhere and derive the (sub-)differential of it in Section \ref{sec:3}. Finally, in Section \ref{sec:4} we demonstrate the benefits in numerical experiments.

\section{Matrix Regularization} \label{sec:2}
Let $S\in\R^{n\times m}$ be a matrix and let $\nu=\min\{n,m\}.$ Define $\sigma(S):\R^{n\times m} \rightarrow \R^{\nu}$ as the function that maps $S$ to its singular values in decreasing order. Let $\sigma_{\max}(S):=\sigma_1(S)$ and $\sigma_{\min >0}(S) := \min_{i\in \{1,...,\nu\}}\{\sigma_i(S) \mid\sigma_i(S)>0\}$ denote the largest and smallest non-zero singular values of $S$, respectively. The rank of a matrix is the number of its non-zero singular values.\\ 
We will denote the singular value decomposition of a matrix $S\in \R^{n\times m}$ as $S=U(\text{Diag}(\sigma(S)))V^T$, where $U\in \R^{n\times n}$ and $V\in \R^{m\times m} $ are unitary matrices and $\text{Diag}(\sigma(S))\in \R^{n\times m}$ denotes a rectangular diagonal matrix with $\text{Diag}(\sigma(S))_{i,i}=(\sigma(S))_i$.\\ 
Throughout this paper we will use the notion of the condition number $\kappa(S)$ of a matrix $S$ as defined in \eqref{kappa}.
 This quantity determines the numerical stability of $S$ by indicating how much the output $Sx$ can change in response to small changes in the input vector $x$.
It is known that $\norm{S}_2=\sigma_{\max}(S)$ and that $\norm{S^\dagger}_2={\sigma_{\min >0}(S)}^{-1}$ \cite{Golub}. This means minimizing $\kappa(S)$ corresponds to minimizing the ratio $\frac{\sigma_{\max}(S)}{\sigma_{\min >0}(S)} \geq 1$. While this quantity is perfectly well-defined for arbitrary non-zero matrices, it has drawbacks.\\
Clearly, $\kappa(S)=1$ is equivalent to the situation where all non-zero singular values of $S$ are equal. However, there is no guarantee or control over the amount of non-zero singular values. For example, matrices with only one non-zero singular value trivially attain the minimal condition number $\kappa$. Such matrices are not very useful as one is usually interested in full-rank solutions. Furthermore, the mapping $\kappa:\R^{n\times m}\rightarrow \R$ is discontinuous whenever a singular value approaches zero (Appendix \ref{app:disc}), and therefore does not allow a proper definition of a gradient. To circumvent both mentioned issues, we propose a different quantity instead:
\begin{align} \label{eq:regularizer}
    r(S) := \frac{1}{2}\norm{S}_2^2-\frac{1}{2\nu}\norm{S}_F^2,
\end{align}
where $\norm{\cdot}_F$ denotes the Frobenius-norm for matrices.
In the following theorem we show that matrices that minimize $r$ also minimize $\kappa$ and additionally have full rank, both desirable features for numerically stable matrices.
\begin{thm} \label{thm:Equiv}
    For any $S\in \R^{n\times m}$ the regularizer $r(S)$ defined in Eq. \ref{eq:regularizer} is non-negative. If $S\neq 0$, then $r(S) = 0$ if and only if $S$ has full rank and $\kappa(S)=1$.
\end{thm}
\begin{proof}
Since $\norm{S}_F^2 =\sum_{i=1}^\nu\sigma_i^2(S) $ \cite{Golub}, we observe that
\begin{align} \label{Th:Eq:Trace}
    \norm{S}_2^2 - \frac{1}{\nu}\norm{S}_F^2=  \sigma_{\max}^2(S) - \frac{1}{\nu} \sum_{i=1}^\nu\sigma_i^2(S) \geq 0,
\end{align}
as $\sigma_{\max}(S)$  is the largest singular value and, hence, the difference between $\sigma^2_{\max}(S)$ and the average value of the squared singular values of $S$ is always non-negative, thus $r(S)\geq 0$. It is straightforward to see that \eqref{Th:Eq:Trace} attains $0$ if and only if $S$ has $\nu$ singular values that are all equal. Since for full rank matrices $S$ all $\nu$ singular values are non-zero, $\kappa(S)=1$ is equivalent to $\sigma_{\max}(S) \equiv \sigma_i(S)>0$ for all $i\in \{1,...,\nu\}$, which concludes the result.
\end{proof}
As noticed before, the condition number is discontinuous and approaches $+\infty$ if a singular value is approaching zero. Thus, finding a tight connection between the regularizer $r$ and the condition number $\kappa$ becomes challenging. The following theorem provides a relationship between the regularizer and the condition number, capturing the divergent behavior of the condition number, whenever $\sigma_{\min>0}(S)$ vanishes. A proof can be found in Appendix \ref{app:proofs}. 
\begin{thm} \label{thm:Approximation}
For $S \in \R^{n\times m}$ it holds that
\begin{align*}
    \kappa(S) \leq e^{{\nu}{\sigma_{\min>0}(S)^{-2}} r(S)}.
\end{align*}
\end{thm}
 This theorem shows that for small $r(S)$, the condition number remains small as long as $\sigma_{\min>0}(S)$ is bounded away from zero. This is important as in practical scenarios, achieving the absolute minimum of regularizers is rarely possible.

We note the close relation to the well-known Tikhonov regularization \cite{Tikh}, also known as $L_2$-regularization or Ridge regression \cite{Ridge}, where $\|S\|_F^2$ is used as a regularizer. While it is effective in improving stability of the problem formulation \cite{Inverse}, it does not directly address the issue of maintaining a low condition number of the solution, which is crucial for numerical stability.

\section{Differential Calculus}\label{sec:3}
In this chapter we derive the (sub)-differential properties and formula of the proposed regularizer $r$ in Eq. \eqref{eq:regularizer}.

First, note that the function $r:\R^{n\times m}\rightarrow \R$ is the difference of two convex functions and therefore non-convex. We will show that it is differentiable \textit{almost everywhere} and its subdifferential exists everywhere. 
For a function \( f: \mathbb{R}^n \to \mathbb{R} \) the subdifferential at a point \( \bar{x} \in \mathbb{R}^n \) is simply denoted by \( \partial f(\bar{x}) \).  In this work, we use the well-known Mordukhovich subdifferential $\partial_M$, also known as the limiting subdifferential (Appendix \ref{app:sub}).

\clearpage
Intuitively, this subdifferential generalizes the concept of subgradients to non-convex non-smooth functions, capturing the behavior at points where they are not differentiable by considering limits of gradients of nearby smooth points. The Mordukhovich subdifferential serves as a unifying framework for subdifferentials, accommodating both smooth and non-smooth, convex and non-convex functions. For convex functions it coincides with the convex subdifferential, and for smooth functions it contains only the gradient (Theorems \ref{thm:app:coinc} and \ref{thm:app:coincDiff}). For this reason and to enhance clarity throughout the main part of the paper, we will use the notation $\partial$ for all subdifferentials at points, where the gradient is not unique.

In the following theorems we derive the (sub-)differential for the regularizer $r$, using classical results from convex analysis and the work of A.S. Lewis on the differentiability of univariate matrix functions \cite{Lewis}. The detailed definitions and theorems used are listed in the Appendix \ref{app:sub}.
 
\begin{thm}\label{Thm: Derivative} Let $S\in \R^{n\times m}$ with a singular value decomposition given by $S=U(\text{Diag}(\sigma(S))V^T$. Let $u_i$ and $v_i$ denote the respective column vectors of $U$ and $V$.

If $\sigma_1(S) > \sigma_i(S)$ for $i>1$, then $r$ is differentiable at $S$:  
\begin{align*}
    \nabla r(S) = \sigma_{\max}(S)u_1v_1^T - \frac{1}{\nu}S.
\end{align*}
Otherwise, the subdifferential of $r$ at $S$ is given by:
\begin{align*}
    \partial r(S) = \text{conv} \{\sigma_{i}(S)u_iv_i^T \mid  i:\sigma_i(S) = \sigma_{\max}(S) \} - \frac{1}{\nu}S.
\end{align*}
\end{thm}
The following results underscores the practicality of our regularizer. It shows that applying gradient descent steps with respect to our regularizer decreases the condition number of the updated matrix. A proof is found in Appendix \ref{app:proofs}.  
\begin{thm}\label{thm:DescentStep} Assume that for $S \in \R^{n\times m}$ the largest singular value is unique, i.e. $\sigma_1(S) > \sigma_i(S)$ for $i \in \{2,\dots,\nu\}$.  Then there exists a $L\in (0,1]$, s.t. for $S'=S-\lambda\nabla r(S)$ with $\lambda \in (0,L]$ it holds $$\kappa(S') < \kappa(S).$$
\end{thm}

\begin{remark}
For a weight matrix at random initialization, $r$ is differentiable almost surely. This follows from the fact that the space of real symmetric matrices with at least one repeated eigenvalue has codimension 2 in the space of all real symmetric matrices \cite{Tao}. Consequently, for a matrix $S$ whose entries are i.i.d. samples from a typical continuous distribution, such as Gaussian or uniform, the matrix $S^TS$ has distinct eigenvalues with probability one. Therefore, $S$ almost surely has distinct singular values.
\end{remark}




\section{Numerical Experiments}\label{sec:4}
We demonstrate the functionality and the benefits of the proposed regularizer through a series of numerical experiments, which are intentionally kept basic to illustrate the core concepts. A comparative analysis with Tikhonov regularization of the results for the experiments is provided in Appendix \ref{app:tikh}. A Python implementation can be accessed via the following link: \href{https://github.com/danedane-haider/Almost-Smooth-Sailing}{github.com/danedane-haider/Almost-Smooth-Sailing}.

\subsection{Basic Functionality}\label{sec:basic}
Introducing regularization inevitably leads to a trade-off between the main optimization goal and achieving the desired regularization. To illustrate the impact of our regularizer, let us consider the matrix least-squares problem
\begin{align}\label{eq:LQM}
    \min_{W\in \R^{n\times m}} \norm{WX-Y}_F^2 + \lambda \cdot r(W),
\end{align}
where $X \in \R^{m\times d}$ and $Y \in \R^{n\times d}$ are fixed. The parameter $\lambda\geq 0$ is our regularization parameter, which controls the trade-off between the objective and the regularization. 

For the experiment we choose $n=20$, $m=50$ and $d=100$, and let $X$ and $Y$ be Gaussian random matrices. We employ gradient descent with the gradient formulas provided in Theorem \ref{Thm: Derivative} for $10^5$ steps, compare the results for different values of $\lambda $ and denote the resulting matrices by $W_\lambda$. We repeat this experiment 10 times for each $\lambda\in\{0,10,...,100\}$ value.

In Figure \ref{fig:lsm} we plot the ratio of the (approximation) errors $\frac{\norm{W_\lambda X-Y}_F}{\norm{W_0X-Y}_F}$ and the condition number $\kappa(W_\lambda)$. We observe that increasing $\lambda$ leads to an slight increase in the error and a sharp decrease in the condition number, as expected. For large values of $\lambda$ the condition number becomes nearly optimal, while the error increases by less than $6\%$ compared to the non-regularized case.
\begin{figure}[t]
    \centering
    \includegraphics[width=\columnwidth]{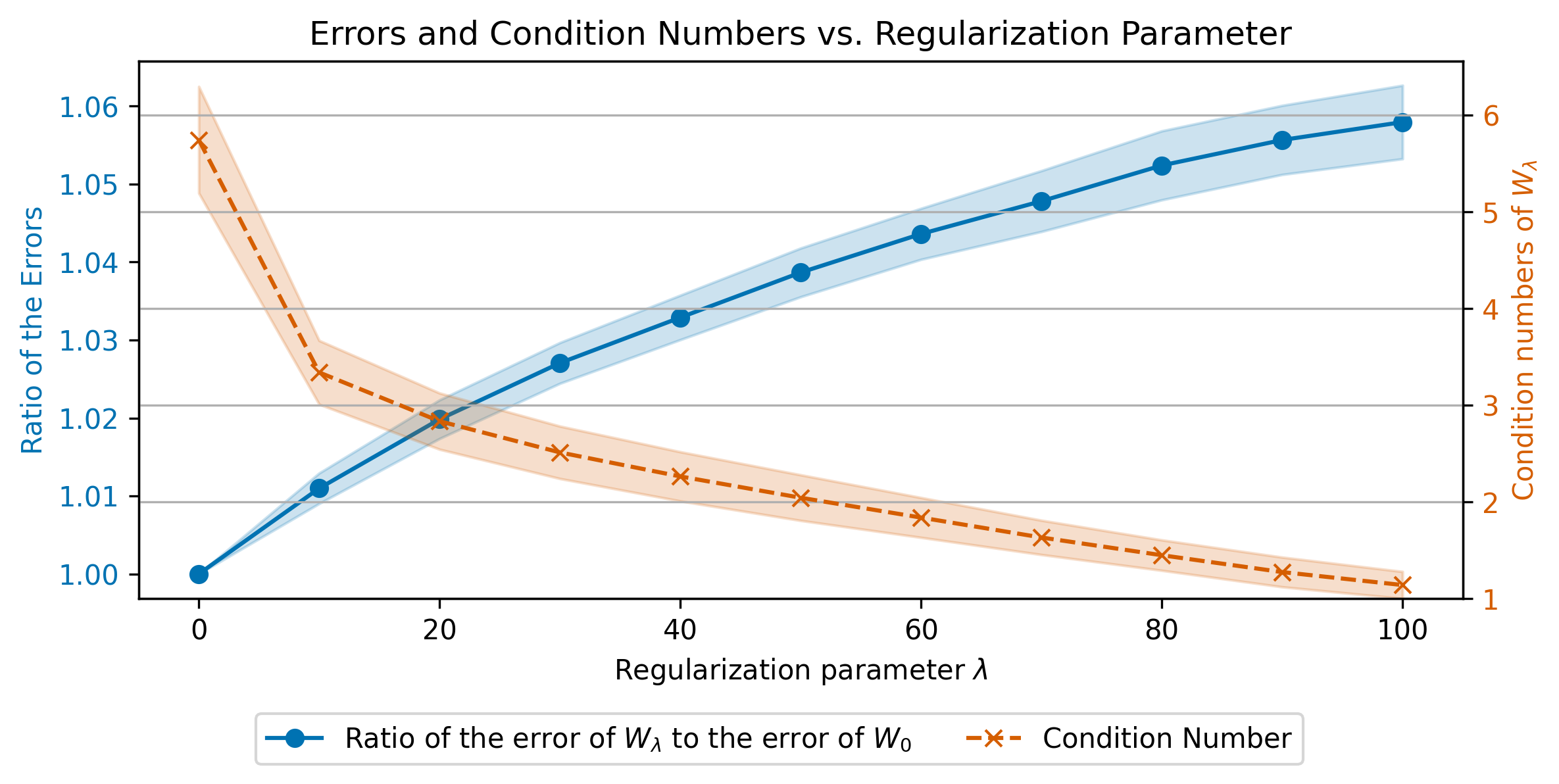}
    \caption{Results of least-squares minimization of \eqref{eq:LQM} after $10^5$ iterations for different regularization parameter $\lambda$ values}
    \label{fig:lsm}
\end{figure}

In the following experiments we integrate the regularizer into the training of a neural network on MNIST \cite{lecun2010mnist}, demonstrating the benefits of a well-conditioned model in the presence of noise.

\subsection{Noisy MNIST Classification}
For a proof-of-concept approach, we choose the classifier model $\Phi$ to be a small neural network with two dense layers, the first one with sigmoid activation, the second one with soft-max. To stabilize the weight matrix of the first layer, $W_1 \in \R^{2048\times 784}$, we perform empirical risk minimization (ERM) using a regularized cross-entropy loss
\begin{equation}
    \mathcal{L}(x;\Phi) = -\sum_{i=1}^{10} y[i] - \log(\Phi(x)[i]) + \lambda \cdot r(W_1),
\end{equation}
where $y[i]$ and $\Phi(x)[i]$ are $i$-th components of the target and predicted label vector, respectively. Optimization is done for $50$ epochs using Adam with a learning rate of $0.0001$.

Table \ref{tab:class} shows the performances of the same model, trained with different values of $\lambda$, and tested on unseen data with added Gaussian noise for different signal-to-noise ratios: $\infty,1$, and $0.5$. In the absence of noise (SNR $\infty$) one can clearly see the influence of the regularizer on the classification performance in terms of the mentioned trade-off. When increasing the noise level the better conditioned models start outperforming the others due to the induced robustness properties. At initialization the condition number is $4.28$.

Clearly, the final goal is to achieve both, best performance for all noise levels and optimal stability. However, it needs further research to determine if and when this is possible.

\begin{table}[t]
  \centering
  \begin{tabular}{@{}l|llllll@{}}
    \toprule
    \rowcolor{lightgray}
     $\lambda$ & $\kappa(W_1)$ & SNR $\infty$ & SNR 1 & SNR 0.5\\
    \midrule
     $0$ & 43.12 & \textbf{98.42 \%} & \textbf{93.80 \%} & 71.91 \% \\
     $10^{-3}$ & 9.43 & \textbf{98.38 \%} & 91.72 \% & 62.51 \%\\
     $10^{-2}$ & 4.62 & 98.11 \% & 91.95 \% & 63.61 \%\\
     $10^{-1}$ & \textbf{1.45} & 96.77 \% & \textbf{93.68 \%} & \textbf{74.27 \%} \\
     $1$ & \textbf{1.53} & 96.50 \% & 92.61 \% & \textbf{84.25} \%\\
    \bottomrule
  \end{tabular}
  \caption{The table presents the condition numbers of the first network layer ($\kappa(W_1)$) and the classification accuracy on the test set with different SNRs: $\infty$ (no noise), $1$, and $0.5$. These results are compared for different values of the hyperparameter $\lambda$.
  The more noise is present (low SNR), the more effective the regularizer is.}
  \label{tab:class}
\end{table}

\subsection{Denoising MNIST}
For denoising we use a basic autoencoder architecture with two dense layers in the encoder and decoder, respectively. Let the weight matrices in order of application be denoted by $E_1,E_2,D_2,D_1$. We set $E_1,D_1^\top\in \R^{256\times 784}$ and $E_2,D_2^\top\in \R^{32\times256}$ and use ReLU activation for all except the last layer, which uses a sigmoid activation. For training, we perform ERM with respect to the regularized $\ell_2$ loss
\begin{align}\label{eq:loss2}
\begin{split}
    \mathcal{L}(\hat{x};\Phi) = \norm{x-\Phi(\hat{x})}_2   &+ \lambda_1 \cdot (r(E_1) + r(D_1))  \\ &
    + \lambda_2 \cdot (r(E_2) + r(E_2)),
\end{split}
\end{align}
where $\hat{x}$ are noisy versions of $x$ with a SNR of 1.
The optimizer is Adam with a learning rate of $0.05$ for $50$ epochs.

Table \ref{tab:denoise} shows the condition numbers of all weight matrices of the naively trained autoencoder ($\lambda_1=\lambda_2=0$) and the regularizer one ($\lambda_1=0.05, \lambda_2=0.01$). Upon initial observation, it becomes immediately apparent that the condition numbers of the outer weight matrices are absurdly high ($\approx 600!$), indicating severe numerical instabilities.
Figure \ref{fig:denoising} illustrates the denoising performance, qualitatively. It is noteworthy that the naively trained model appears to encounter difficulties in reconstructing the images at all, whereas the regularized model demonstrates remarkable performance, even in the presence of substantial noise.

We are aware that there exist more optimal architectures for both classification \cite{Classification} and denoising \cite{Denoising}, however, the objective of this experiments is to illustrate the advantages of the proposed regularizer in a fundamental setting.

\begin{table}[t]
  \centering
  \begin{tabular}{@{}ll|lllll@{}}
    \toprule
    \rowcolor{lightgray}
     $\lambda_1$ & $\lambda_2$ & $\kappa(E_1)$ & $\kappa(E_2)$ & $\kappa(D_2)$ & $\kappa(D_1)$\\
    \midrule
     $0$ & $0$ & 604.43 & 59.58 & 30.76 & 102.82 \\
     $0.1$ & $0.005$ & 23.39 & 1.08 & 1.07 & 12.35 \\
    \bottomrule
  \end{tabular}
  \caption{Condition numbers of all weight matrices of the trained autoencoder for image denoising, with and without regularization. It is noticeable that the weight matrices of the outer layers ($E_1,D_1$) are absurdly high if not regularized.}
  \label{tab:denoise}
\end{table}

\section{Conclusion}
Traditional approaches to regularization, such as Tikhonov regularization, often fail to directly address the condition number, which is crucial for maintaining numerical stability. This paper introduces a novel almost everywhere differentiable regularizer that enhances numerical stability in neural networks by promoting low condition numbers in their weight matrices. 
Through a theoretical analysis and a series of numerical experiments, we proved and demonstrated the properties of this regularizer. In a noisy classification and a denoising task using the MNIST dataset, models that are trained with the proposed regularizer exhibit significantly lower condition numbers of the weight matrices and show robustness against noise successfully compared to models that are trained without regularization. 


\section*{Acknowledgment}
The authors thank the reviewers for their time reviewing and their feedback. The work of R. Nenov and P. Balazs was supported by the FWF project NoMASP (P 34922). D. Haider is recipient of a DOC Fellowship of the Austrian Academy of Sciences at the Acoustics Research Institute (A 26355).

\clearpage
\printbibliography

\newpage
\appendix
\onecolumn
\begin{center}
    \Large 
    \textbf{Appendix}
\end{center}

\section{On the Discontinuity of the Condition Number}\label{app:disc}

\begin{example}\label{thm:ex} This example shows that the condition number $\kappa(S):=\norm{S}_2\norm{S^\dagger}_2$ is not a continuous mapping \cite{CounterExample}. 
    Let 
    \[
    A = \begin{pmatrix}
    1 & 0 \\
    0 & 0
    \end{pmatrix}
    \quad \text{and} \quad 
    E = \begin{pmatrix}
    0 & 0 \\
    0 & 1
    \end{pmatrix}.
    \]
    For each $1>\epsilon >0$ we have
    \[
    (A + \epsilon E)^{\dagger} = \begin{pmatrix}
    1 & 0 \\
    0 & \epsilon
    \end{pmatrix}^{\dagger} = \begin{pmatrix}
    1 & 0 \\
    0 & \epsilon^{-1}
    \end{pmatrix}.
    \]
    Therefore $\kappa(A+\epsilon E)=\norm{A+\epsilon E }_2\norm{(A + \epsilon E)^{\dagger}}_2=\epsilon^{-1}$. Hence $A + \epsilon E \to A$ as $\epsilon \to 0$, but $\lim_{\epsilon \to 0} \kappa(A+\epsilon E)$ does not exist, even though $\kappa(A)=1.$
\end{example}

\section{Essentials from Subdifferential Calculus}\label{app:sub}
In this part of the appendix, we include all definitions and results used in the paper to make the document self-contained and the derivation of the Theorems precise and unambiguous.

\subsection{Convex Subdifferential \cite{rockafellar}}
\begin{definition}[Proper function] An extended value function $f:\R^n\to \R\cup \{+\infty\}$ is called proper, if its domain $ \dom(f) := \{ x\in \R^n: f(x) < + \infty \} $ is not empty.  
\end{definition}

\begin{definition}[Convex Subdifferential]
    Let \( f: \R^n \to \R \cup \{+\infty\} \) be a proper convex function. The \emph{(convex) subdifferential} of \( f \) at \( x \in \dom(f)  \) is defined as
    \[
    \partial f(x) = \{ v \in \R^n \mid f(y) \geq f(x) + \skal{v, y - x} \ \forall y \in \R^n \}.
    \]
    The elements of \( \partial f(x) \) are called \emph{subgradients} of \( f \) at \( x \).
\end{definition}

\begin{remark}
    If \( x \notin \dom(f) \), then \( \partial f(x) = \emptyset \).
\end{remark}

\subsection{Mordukhovich Subdifferential  \cite{mordukhovich2018variational}}
\begin{definition}[Mordukhovich Subdifferential]
    Let \( f: \R^n \to \R \cup \{+\infty\} \) be a lower semicontinuous function. The \emph{Mordukhovich (or limiting) subdifferential} of \( f \) at \( x \in \dom(f) \) is defined as
    \[
    \partial_M f(x) = \left\{ v \in \R^n \mid \exists x^k \to x, \, v^k \to v \text{ with } v^k \in \hat{\partial} f(x^k) \text{ and } f(x^k) \to f(x) \right\},
    \]
    where \( \hat{\partial} f(x) \) denotes the \emph{Fréchet subdifferential} of \( f \) at \( x \), defined by
    \[
    \hat{\partial} f(x) = \left\{ v \in \R^n \mid \liminf_{y \to x} \frac{f(y) - f(x) - \skal{v, y - x}}{\norm{y - x}} \geq 0 \right\}.
    \]
\end{definition}

\begin{remark}
    The Mordukhovich subdifferential generalizes the concept of subgradients to non-convex functions and is particularly useful in variational analysis and optimization.
\end{remark}
\subsubsection{Coincidence of Subdifferentials \cite{mordukhovich2018variational}}
\begin{theorem} \label{thm:app:coinc}
    Let \( f: \R^n \to \R \cup \{+\infty\} \) be a proper, lower semicontinuous, and convex function. Then, for any \( x \in \R^n \),
    \[
    \partial f(x) = \partial_M f(x).
    \]
\end{theorem}
\begin{theorem}\label{thm:app:coincDiff}
    Let \( f: \R^n \to \R \) be a differentiable function at \( x \in \R^n \). Then,
    \[
    \partial f(x) = \partial_M f(x) = \{ \nabla f(x) \}.
    \]
\end{theorem}
\begin{theorem} \label{thm:app:single}
    Let \( f: \R^n \to \R \cup \{+\infty\} \) be a proper, lower semicontinuous, and convex function. If the convex subdifferential \( \partial f(x) \) at \( x \in \R^n \) is a singleton, say \( \partial f(x) = \{v\} \), then \( f \) is differentiable at \( x \) and \( \nabla f(x) = v \).
\end{theorem}
\subsubsection{Mordukhovich Subdifferential of the Sum of Functions \cite{mordukhovich2018variational}}
\begin{theorem}\label{thm:app:sum}
    Let \( f: \R^n \to \R \cup \{+\infty\} \) be a proper, lower semicontinuous, and convex function, and let \( g: \R^n \to \R \) be a differentiable function. Then, for any \( x \in \dom f(x)\):
    \[
    \partial_M (f + g)(x) = \partial f(x) + \nabla g(x).
    \]
\end{theorem}

\subsection{Rules of Differentiation \cite{mordukhovich2018variational}}
\begin{definition}[Convex Hull]
    Let \( S \) be a subset of \( \R^n \). The \emph{convex hull} of \( S \), denoted by \( \text{conv}(S) \), is the smallest convex set containing \( S \). It can be defined as:
    \[
    \text{conv}(S) = \left\{ \sum_{i=1}^{k} \lambda_i x_i \ \middle|\ x_i \in S, \ \lambda_i \geq 0, \ \sum_{i=1}^{k} \lambda_i = 1, \ k \in \N \right\}.
    \]
    In other words, the convex hull of \( S \) is the set of all convex combinations of points in \( S \).
\end{definition}
\begin{theorem}\label{thm:app:max}
    Let \( \{f_i\}_{i \in I} \) be a finite family of proper, lower semicontinuous, convex, and differentiable functions from \( \R^n  \to  \R \). Define \( f: \R^n \to \R \) by
    \[
    f(x) = \max_{i \in I} f_i(x).
    \]
    Then, for any \( x \in \R^n \),
    \[
    \partial_M f(x) = \text{conv} \left\{ \nabla f_i(x) \mid i \in I_x \right\},
    \]
    where \( I_x = \left\{ i \in I \mid f_i(x) = \max_{i\in I}f_i(x) \right\} \) denotes the set of indices, for which $f_i$ attains the largest value at $x$.
\end{theorem}

\subsection{Results on  unitarily invariant matrix functions \cite{Lewis}}
\begin{definition}[Absolutely Symmetric Function]\label{thm:app:abs}
    A function \( f : \R^q \to \R \) is said to be \emph{absolutely symmetric} if \( f(\gamma) = f(\gamma^s) \) for any permutation \( s \) of the components of \( \gamma \) and for any \( \gamma \in \R^q \). Equivalently, \( f \) is absolutely symmetric if
    \[
    f(Q\gamma) = f(\gamma) \quad \text{for all} \ \gamma \in \R^q \ \text{and} \ Q \in \Lambda_q,
    \]
    where \( \Lambda_q \) denotes the set of generalized permutation matrices (matrices with exactly one non-zero entry in each row and each column, that entry being \( \pm 1 \)).
\end{definition}
This means that the function value at $x$ of absolutely symmetric functions is independent of the ordering of the entries of $x$. Recall that $\sigma$ maps a matrix onto its singular values in nonincreasing order. 
\begin{theorem}[Characterization of Convexity]\label{thm:app:conv}
    Suppose that the function \( f : \R^q \to (-\infty, +\infty] \) is absolutely symmetric. Then the corresponding unitarily invariant function \( f \circ \sigma \) is convex and lower semicontinuous on \( \mathbb{C}^{m \times n} \) if and only if \( f \) is convex and lower semicontinuous.
\end{theorem}
\begin{theorem}[Characterization of Subgradients]\label{thm:app:Msubg}
    Suppose that the function \( f : \R^q \to (-\infty, +\infty] \) is absolutely symmetric, and that the \( m \times n \) matrix \( X \) has \( \sigma(X) \) in \( \text{dom}(f) \). Then
    \[
    \partial (f \circ \sigma)(X) = \left\{ U (\text{Diag} \ \mu) V^T \ \middle| \ \mu \in \partial f (\sigma(X)), \ X = U (\text{Diag} \ \sigma(X)) V^T \right\}.
    \]
\end{theorem}

\begin{theorem}[Gradient Formula]\label{thm:app:Mgrad}
    If a function \( f : \R^q \to (-\infty, +\infty] \) is convex and absolutely symmetric then the corresponding convex, unitarily invariant function \( f \circ \sigma \) is differentiable at the \( m \times n \) matrix \( X \) if and only if \( f \) is differentiable at \( \sigma(X) \). In this case,
    \[
    \nabla (f \circ \sigma)(X) = U (\text{Diag} \ \nabla f(\sigma(X))) V^T,
    \]
    for \( X = U (\text{Diag} \ \sigma(X)) V^T \).
\end{theorem}

\newpage
\section{Proofs}\label{app:proofs}
In this section we proof Theorems \ref{Thm: Derivative}, \ref{thm:Approximation}, and \ref{thm:DescentStep} formulated in the main body of the paper.\\

\textbf{Proof of Theorem \ref{Thm: Derivative}}
\begin{proof}
We define $f:\R^\nu\to \R$ and $g:\R^\nu\to \R$ as 
\begin{align*}
    f(x)=\max_{i\in 1\dots\nu} \frac{1}{2}x_i^2, \quad g(x)=\frac{1}{2\nu}\sum_{i=1}^\nu x_i^2.
\end{align*}
One can see that by Definition \ref{thm:app:abs} both $f$ and $g$ are absolutely symmetric and $r(S)=f(\sigma(S))+g(\sigma(S)). $ 
The function $f$ is convex by the rule for pointwise maxima Theorem \ref{thm:app:max} and its subdifferential is given by
\begin{align*}
        \partial f(x) = \text{conv} \{ x_ie_i \mid i: x_i^2 = x_{\max}^2\},
    \end{align*}
    where $e_i$ denotes the $i-th$ standard unit vector and $\max$ denotes the index, for which the largest value is attained.
By the results of Theorem \ref{thm:app:conv} 
    we deduce that $f(\sigma(S))$ is convex and by Theorem \ref{thm:app:Msubg}
    \begin{align*}
        \partial (f\circ \sigma)(S) 
        &= \text{conv} \{\sigma_i(S)u_iv_i^T \mid i: \sigma_i(S) = \sigma_{max}(S)\}. 
    \end{align*}
    The function $g$ is differentiable everywhere and its gradient is given by
    $
        \nabla g(x) = \frac{1}{\nu}x.
    $
    By Theorem \ref{thm:app:Mgrad} we deduce 
    \begin{align*}
        \nabla (g\circ \sigma)( S )&= U(\text{Diag}(\nabla g(\sigma(S)))V^T \\ &= -\frac{1}{\nu} U(\text{Diag}(\sigma(S))V^T = -\frac{1}{\nu}S.
    \end{align*}
    Thus, by Theorem \ref{thm:app:sum}, we derive 
    \begin{align*}
    \partial r(S) = \text{conv} \{\sigma_{i}(S)u_iv_i^T \mid  i:\sigma_i(S) = \sigma_{\max}(S) \} - \frac{1}{\nu}S.
\end{align*} 
If the largest singular value of $S$ is unique, $\partial r(S)$ is a singleton and by Theorem \ref{thm:app:single} therefore differentiable with
\begin{align*}
    \nabla r(S) = \sigma_{\max}(S)u_1v_1^T - \frac{1}{N}S.
\end{align*}
\end{proof}

\textbf{Proof of Theorem \ref{thm:Approximation}}
\begin{proof}
Using the Mean Value Theorem for the logarithm, we derive for $c \in (\sigma_k(S)^2,\sigma_1(S)^2)$
    \begin{align*}
        2\ln{(\kappa(S))}&=\ln{(\kappa(S)^2)}\\&=\ln{(\sigma_1(S)^2)}-\ln{(\sigma_k(S)^2)},\\&= \frac{1}{c} \abs{\sigma_1(S)^2-\sigma_k(S)^2}, \\
        &\leq\frac{1}{\sigma_k(S)^2} (\sigma_1(S)^2-\sigma_k(S)^2).
    \end{align*}
Furthermore estimating the regularizer yields
    \begin{align*}
        2r(S) &= \sigma_1(S)^2 - \frac{1}{\nu}\sum_{i=1}^{k}\sigma_i(S)^2, \\
        &\geq \sigma_1(S)^2-\frac{1}{\nu}\left((k-1)\sigma_1(S)^2+\sigma_k(S)^2 \right),\\
        &\geq \sigma_1(S)^2-\frac{1}{\nu}\left((\nu-1)\sigma_1(S)^2+\sigma_k(S)^2 \right), \\
        &=\frac{1}{\nu}(\sigma_1(S)^2-\sigma_k(S)^2).
    \end{align*}
Combining those two results we conclude:
    \begin{align*}
        \ln{(\kappa(S))} &\leq \frac{\nu}{\sigma_k(S)^2} r(S), \\
        \kappa(S) &\leq e^{\frac{\nu}{\sigma_k(S)^2} r(S)}.
    \end{align*}
\end{proof}

\textbf{Proof of Theorem \ref{thm:DescentStep}}
\begin{proof}
 Let $S=U\text{Diag}(\sigma(S))V^T=\sum_{i=1}^\nu\sigma_i(S)u_iv_i^T$ be the singular value decomposition of $S$. Let $k$ be the index of the smallest non-zero singular value, i.e. $k=\argmin_i\{\sigma_i(S) \mid \sigma_i(S)>0 \text{ for }k\in \{1,\dots,\nu \}\}$. Then it holds that:
    \begin{align*}
        S' &= S - \lambda \nabla r(S) \\ 
        &= \left(1+\frac{\lambda}{\nu}\right) S - \lambda\sigma_1(S)u_1v_1^T \\
        &= \left(1+\frac{\lambda}{\nu}-\lambda\right)\sigma_1(S)u_1v_1^T + \left(1+\frac{\lambda}{\nu}\right)\sum_{i=2}^\nu\sigma_i(S)u_iv_i^T \\
        &= U\text{Diag}(\sigma')V^T, 
    \end{align*}
    where $\sigma' = \left(\left(1+\frac{\lambda}{\nu}-\lambda\right)\sigma_1(S), \left(1+\frac{\lambda}{\nu}\right)\sigma_2(S),\dots, \left(1+\frac{\lambda}{\nu}\right)\sigma_\nu(S)\right)$. Given this decomposition of $S'$, we notice that $\sigma'$ are the singular values of $S'$, but not necessarily in the right order. Therefore we are going to distinct two cases. \smallskip \\
    Since $S$ is assumed to have a unique largest singular value, there exists an $\alpha>1$ s.t. $\sigma_1(S) = \alpha \sigma_2(S)$. Notice that $\kappa(S) \geq \alpha$ holds. Choose $L\in (0,1]$, s.t. for all $\lambda \in (0,L]:$ $1-\frac{\lambda}{1+\frac{\lambda}{\nu}}\geq \frac{1}{\kappa(S)}$.\\
    \textbf{1. Case:} $\frac{1}{\alpha}<1-\frac{\lambda}{1+\frac{\lambda}{\nu}}. $ \\
    By the case distinction we see:
    \begin{align*}
        \left(1+\frac{\lambda}{\nu}-\lambda\right)\sigma_1(S) &> \frac{1}{\alpha} \left(1+\frac{\lambda} 
        {\nu}\right)\sigma_1(S), \\
        &= \left(1+\frac{\lambda} 
        {\nu}\right)\sigma_2(S),
    \end{align*}
     and thus $\sigma_1' > \sigma_2'$. Since $\sigma_i(S)\geq \sigma_j(S)$ for $i<j$, we deduce that $\sigma'$ are exactly the singular values of $S'$ in non-increasing order. Furthermore the amount of non-zero singular values stays the same. Therefore 
     \begin{align*}
         \kappa(S') = \frac{\sigma_1'}{\sigma_k'} &= \frac{ \left(1+\frac{\lambda}{\nu}-\lambda\right)}{ \left(1+\frac{\lambda}{\nu}\right)} \frac{\sigma_1(S)}{\sigma_k(S)} \\
         &= \left(1-\frac{\lambda}{1+\frac{\lambda}{\nu}}\right)\kappa(S).
     \end{align*}
     
     \textbf{2. Case:} $\frac{1}{\alpha}\geq 1-\frac{\lambda}{1+\frac{\lambda}{\nu}}\geq \frac{1}{\kappa(S)}$. \\
     Similarly to the previous case, simple arithmetics show that in this case $\sigma_2'$ becomes the largest singular value of $S'$ and $\sigma_k'$ remains the smallest. Therefore:
     \begin{align*}
         \kappa(S') = \frac{\sigma_2'}{\sigma_k'} = \frac{\sigma_2(S)}{\sigma_k(S)} &= \frac{1}{\alpha} \frac{\sigma_1(S)}{\sigma_k(S)} \\
         &=\frac{1}{\alpha} \kappa(S).
     \end{align*}
     In both cases we see that $\kappa(S')<\kappa(S)$, since $\lambda \in (0,L]$ and $\alpha>1$. 
\end{proof}

\begin{figure}[h]
    \centering
    \begin{subfigure}{0.31\textwidth}
        \centering
        \includegraphics[width=\textwidth]{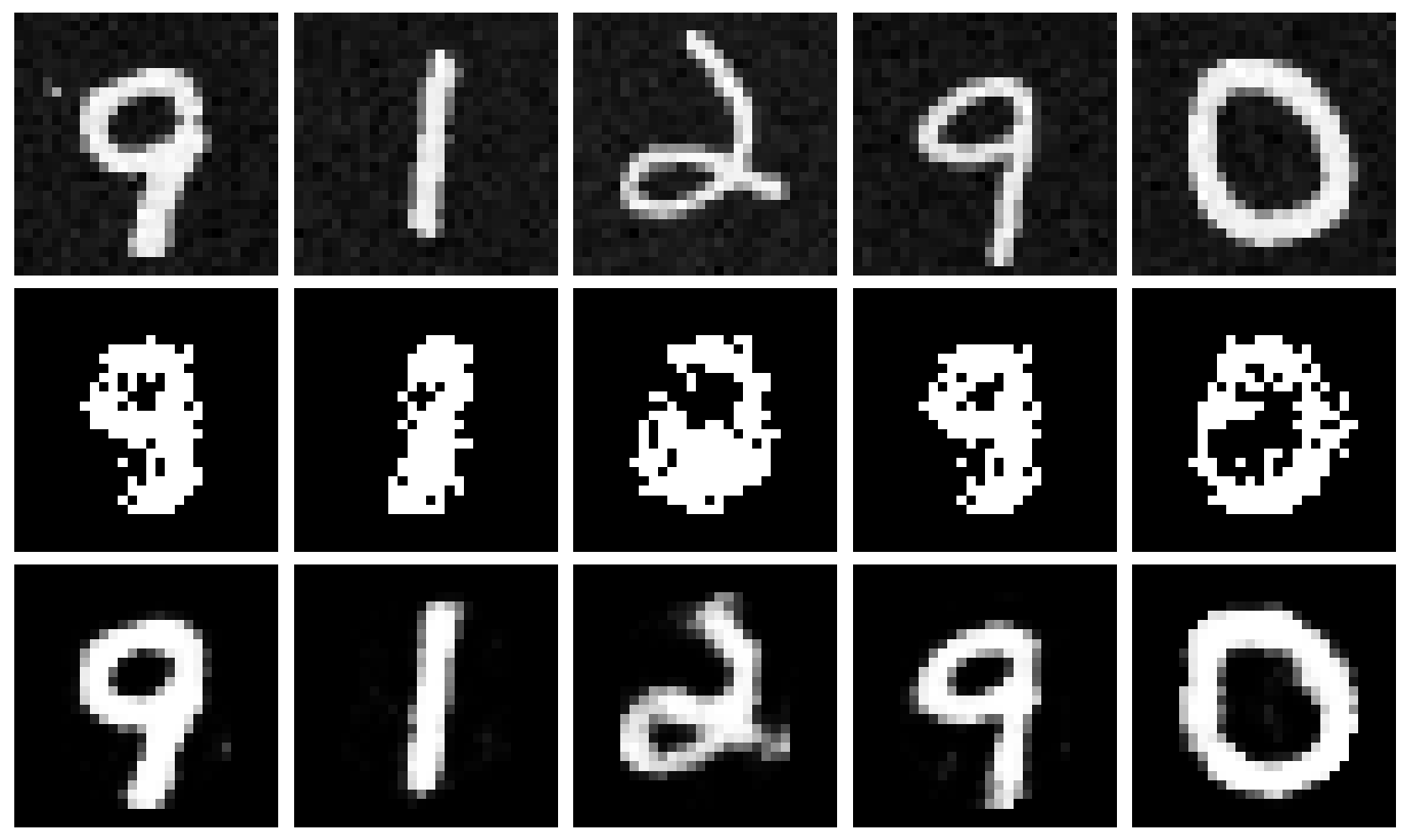}
    \end{subfigure}
    \begin{subfigure}{0.31\textwidth}
        \centering
        \includegraphics[width=\textwidth]{MNIST_denoising.png}
    \end{subfigure}
    \begin{subfigure}{0.31\textwidth}
        \centering
        \includegraphics[width=\textwidth]{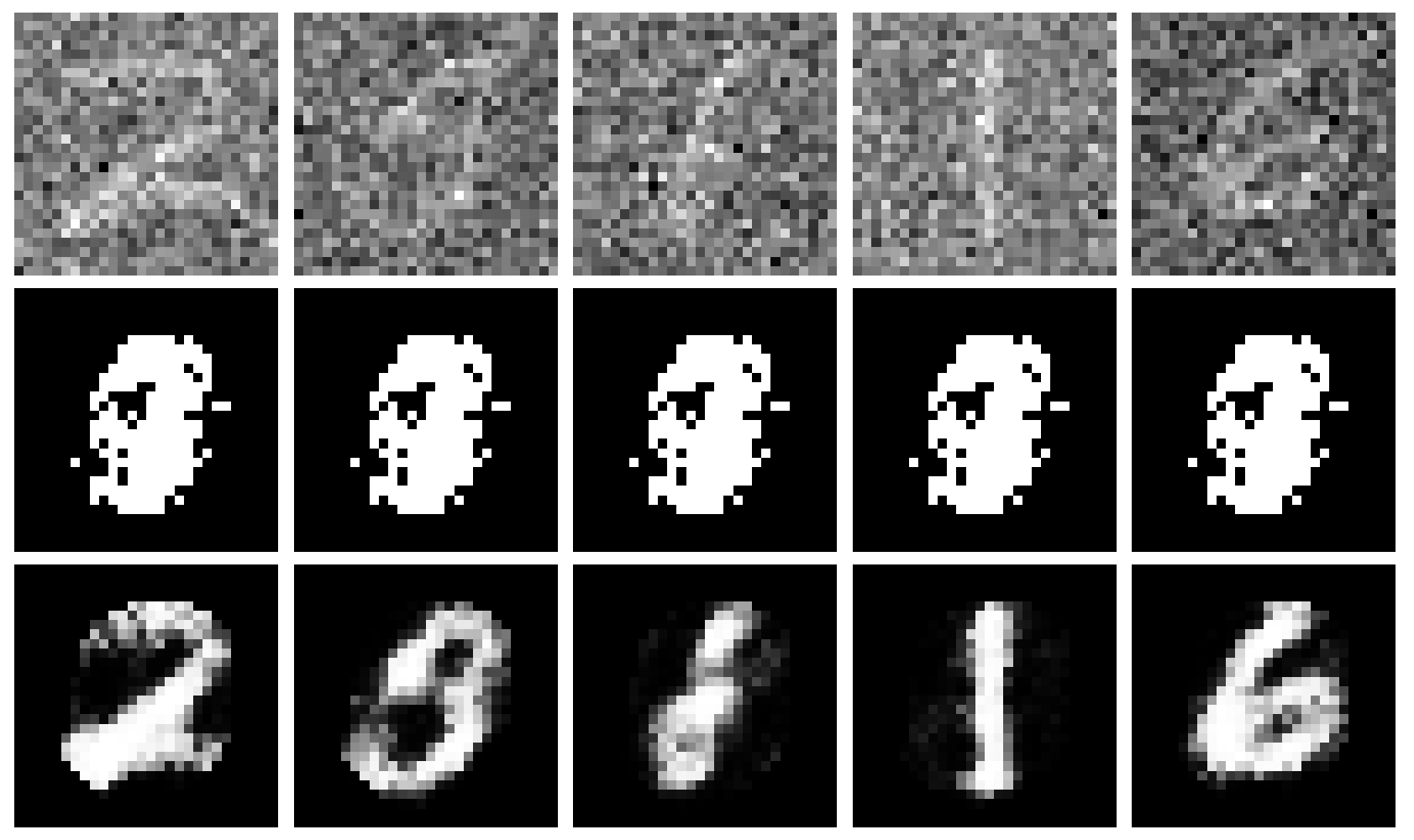}
    \end{subfigure}
    \caption{MNIST denoising results with (bottom) and without regularization (mid) with three different SNRs, from left to right: $10,1,0.5$. Already with reconstructing the images from almost no noise, the non-regularized autoencoder struggles.
    Due to the high condition numbers in the network, the output is very sensitive to perturbations in the input, resulting in the network being unable to learn properly.
    }
    \label{fig:denoising2}
\end{figure}

\section{Discussion on Tikhonov Regularization}\label{app:tikh}
\begin{figure}[h]
    \centering
    
    \centering
    \begin{subfigure}{0.5\textwidth}
        \centering
        \includegraphics[width=\textwidth]{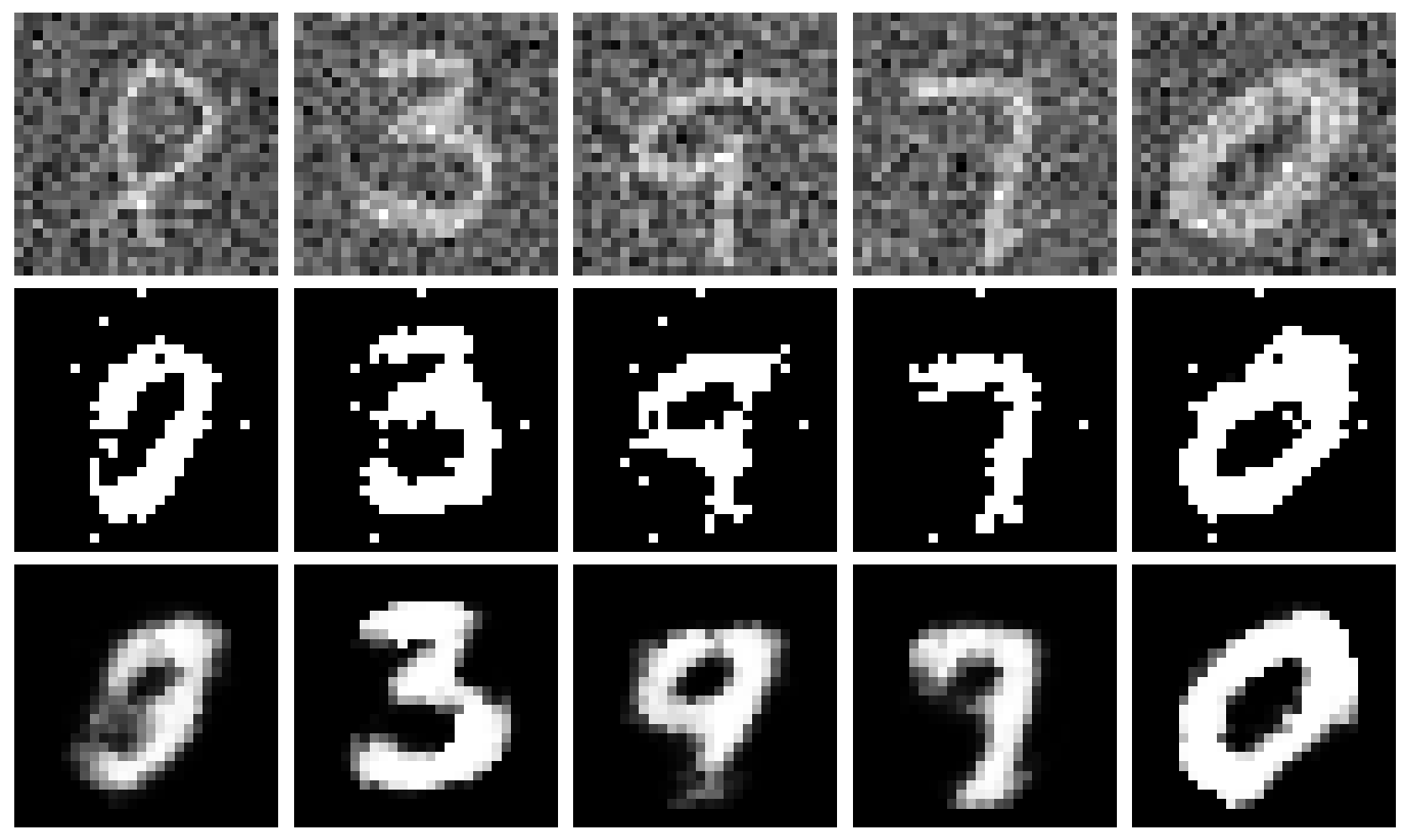}
    \end{subfigure}
    \begin{subfigure}{0.45\textwidth}
        \centering
        \includegraphics[width=\textwidth]{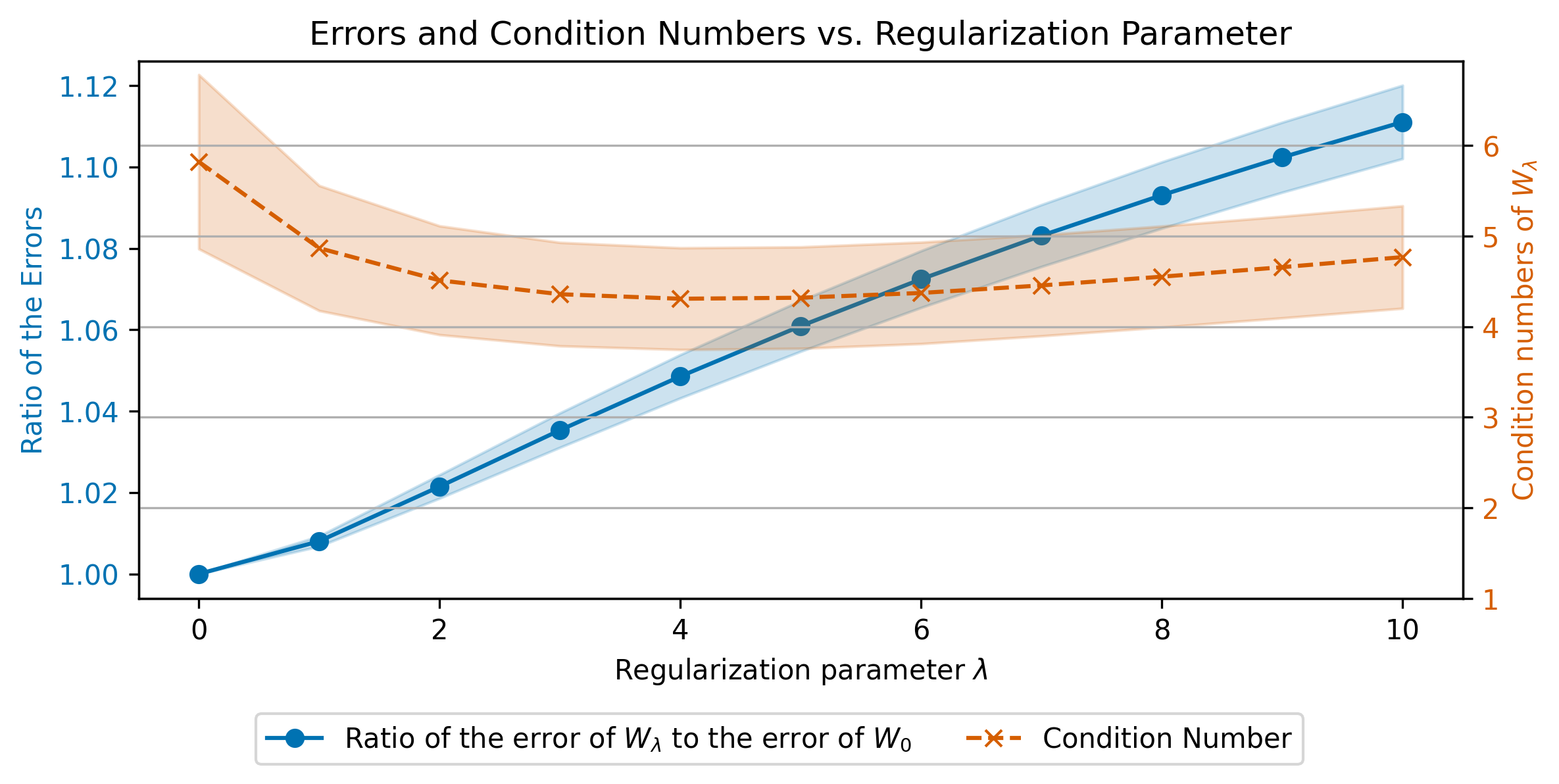}
    \end{subfigure}

    \caption{Left: Results of MNIST denoising with SNR 1 with Tikhonov regularization for two different sets of parameters. Mid: $\lambda_1=0.01$, $\lambda_2=0.0001$. Bottom: $\lambda_1=1$, $\lambda_2=0.01$. 
    Right: Results of least-squares minimization of \eqref{eq:LQM} after $10^5$ iterations with Tikhonov Regularizer for different regularization parameter $\lambda$ values}
    \label{fig:denoisingtik}
\end{figure}
Tikhonov Regularization is a well-known method used to stabilize the solutions of ill-posed problems and to prevent overfitting in machine learning models \cite{Calvetti2003, MLBOOK}. This method adds a regularization term to the loss function, which penalizes large coefficients in the solution, and helps in balancing between fitting the data and simplicity in the solution. 

We include this discussion on Tikhonov regularization due to its widespread recognition and frequent application as a common regularizer, although its direct impact on the condition number of the weight matrices remains, to the best of the authors knowledge, ambiguous in the existing literature.

Mathematically, Tikhonov regularization adds to the standard least squares problem a penalty term proportional to the square of the norm of the coefficients, in our case:
\begin{align*}
    \min_{W\in \R^{n\times m}} \norm{WX-Y}_F^2+\lambda \cdot \norm{W}_F^2,
\end{align*}
where $\lambda$ is the regularization parameter controlling the trade-off between model fit and the magnitude of coefficients \cite{Tikh}.

\subsection{Impact on Condition Number}
While Tikhonov regularization is effective in promoting stability in poorly behaved optimization problems and preventing overfitting, it does not directly address the condition number of the weight matrices.
In the context of neural networks, using the Frobenius norm as a regularizer can help in reducing the overall magnitude of the weights but might not sufficiently control the condition number, as we will see in the following simulation.

We repeat the simulations from Section \ref{sec:basic} for Tikhonov Regularization, where \( r(W) = \norm{W}_F^2 \) in \eqref{eq:LQM}. The results are shown in Figure \ref{fig:denoisingtik}. We observe that Tikhonov regularization slightly reduces the condition number, but not significantly, before the (approximation) error \( \norm{W_\lambda X - Y}_F \) becomes too large compared to the non-regularized error \( \norm{W_0 X - Y}_F \).

\subsection{Numerical Experiments with Tikhonov Regularization}
We conducted numerical experiments on the same problems as in Section \ref{sec:4} to compare the performance of Tikhonov regularization with the proposed method. The results are summarized below:
\subsubsection{\textbf{MNIST Classification}}
Table 3 shows the classification accuracy and condition numbers for different levels of noise (SNR) using Tikhonov regularization. As $\lambda$ increases, the condition numbers of the weight matrices increase dramatically, indicating potential numerical instability. Yet, we see that it performs well at a medium noise level (SNR 1) and outperforms the baseline, as well as our proposed regularizer in all settings of $\lambda$. For the high noise level (SNR 0.5), however, it falls back again. 
\subsubsection{\textbf{MNIST Denoising}}
Figure 4 illustrates the denoising performance with low (mid) and high (bottom) influence by means of the values $\lambda_1,\lambda_2$. Table 4 shows the corresponding condition numbers. We see that Tikhonov regularization fails to maintain low condition numbers consistently across all layers, which explains the worse denoising performance compared to the proposed regularizer.

\subsection{Conclusion}
While Tikhonov regularization offers benefits in terms of regularizing the magnitude of weights and robustness to a medium level of noise, it does not effectively control the condition number of weight matrices, leading to potential numerical instability. The proposed regularizer in this paper addresses this limitation by specifically targeting the condition number, thus enhancing the robustness and stability of neural networks, especially in very noisy environments.

\begin{table}[t]
  \centering
  \begin{minipage}[t]{0.49\textwidth}
    \centering
    \begin{tabular}{@{}l|llll@{}}
      \toprule
      \rowcolor{lightgray}
      $\lambda$ & $\kappa(W_1)$ & SNR $\infty$ & SNR $1$ & SNR $0.5$ \\
      \midrule
      $0$ & \textbf{43.29} & \textbf{98.42 \%} & 93.80 \% & 71.91 \% \\
      $10^{-3}$ & \textbf{56.43} & \textbf{98.31 \%} & 94.34 \% & 74.24 \% \\
      $10^{-2}$ & 242.39 & 98.25 \% & 94.83 \% & 76.94 \% \\
      $10^{-1}$ & 4044.34 & 98.15 \% & \textbf{95.10 \%} & \textbf{78.67 \%} \\
      $1$ & 4020.47 & 97.68 \% & \textbf{95.02 \%} & \textbf{82.79 \%} \\
      \bottomrule
    \end{tabular}
    \caption{Results when using Tikhonov regularization for MNIST classification.}
    \label{tab:classtik}
  \end{minipage}%
  \hfill
  \begin{minipage}[t]{0.49\textwidth}
    \centering
    \begin{tabular}{@{}ll|llll@{}}
      \toprule
      \rowcolor{lightgray}
      $\lambda_1$ & $\lambda_2$ & $\kappa(E_1)$ & $\kappa(E_2)$ & $\kappa(D_2)$ & $\kappa(D_1)$ \\
      \midrule
      $1$ & $0.01$ & 12.80 & 1.43 & 2.65 & 6061.53 \\
      $0.1$ & $0.001$ & 1277.21 & 31641.66 & 20358.70 & 5347.09 \\
      \bottomrule
    \end{tabular}
    \caption{Condition numbers of the weight matrices when using Tikhonov regularization for denoising in two parameter settings.}
    \label{tab:denoisetik}
  \end{minipage}
\end{table}

\end{document}